\documentclass[times,twocolumn,final,authoryear]{article}

\usepackage{framed,multirow}

\usepackage{amssymb}
\usepackage{latexsym}

\usepackage{url}
\usepackage{xcolor}
\definecolor{newcolor}{rgb}{.8,.349,.1}
\usepackage{amsmath,graphicx}
\usepackage[T1]{fontenc} 
\usepackage[latin1]{inputenc}
\usepackage[frenchb,english]{babel}

\usepackage[capitalize,noabbrev]{cleveref} 
\crefname{equation}{equation}{equations}
\providecommand*{\fullref}[1]{\hyperref[{#1}]{\cref*{#1}. \nameref*{#1}}}
\providecommand*{\Fullref}[1]{\hyperref[{#1}]{\Cref*{#1}. \nameref*{#1}}}

\usepackage{amssymb} 
\usepackage{amsmath} 
\usepackage{amsfonts} 
\usepackage{mathrsfs}
\usepackage{amsthm}
\usepackage{kmath} 			

\usepackage{soul}		         

\usepackage{graphicx}
\usepackage{color}
\usepackage{pgfplots}
\usepackage{epsfig}
\usepackage{epstopdf}

\def\a{{\mathbf a}}
\def\b{{\mathbf b}}

\def\e{{\mathbf e}}
\def\f{{\mathbf f}}

\def\o{{\mathbf o}}
\def\p{{\mathbf p}}

\def\u{{\mathbf u}}

\def\w{{\mathbf w}}
\def\x{{\mathbf x}}
\def\y{{\mathbf y}}
\def\z{{\mathbf z}}
\def\0{{\mathbf 0}}
\def\1{{\mathbf 1}}

\def\A{{\mathbf A}}
\def\B{{\mathbf B}}

\def\G{{\mathbf G}}

\def\I{{\mathbf I}}
\def\J{{\mathbf J}}
\def\K{{\mathbf K}}

\def\M{{\mathbf M}}
\def\N{{\mathbf N}}

\def\T{{\mathbf T}}

\def\aa{{\boldsymbol{a}}}

\def\ss{{\boldsymbol s}}

\def\Ac{{\mathcal A}}
\def\Bc{{\mathcal B}}

\def\Qc{{\mathcal Q}}

\def\Nbb{{\mathbb N}}

\def\Rbb{{\mathbb R}}

\def\eg{\textit{e.g.},\xspace}

\newkmacro\sign[1][{\cdot}]{\mathrm{sign}\left({#1}\right)} 
\newkmacro\interior[1][{\cdot}]{\mathrm{int}\left({#1}\right)}
\newkmacro\bd[1][{\cdot}]{\mathrm{bd}\left({#1}\right)}
\newkmacro\conv[1][{\cdot}]{\mathrm{conv}\left({#1}\right)}
\newkmacro\rank[1][{\cdot}]{\mathrm{rank}\left({#1}\right)}
\newkmacro\spa[1][{\cdot}]{\mathrm{span}\left({#1}\right)}
\newkmacro\supp[1][{\cdot}]{\mathrm{supp}\left({#1}\right)}
\newkmacro\Ker[1][{\cdot}]{\mathrm{ker}\left({#1}\right)}

\newtheorem{definition}{Definition}

\newtheorem{proposition}{Proposition}
\newtheorem{lemma}{Lemma}
\newtheorem{theorem}{Theorem}

\def\o#1{\overline{#1}}

\newcommand{\reffig}[1]{\ref{#1}}

\usepackage{tabularx}
\usepackage{multirow}
\usepackage{times}
\usepackage{url}
\usepackage{epsfig}
\usepackage{subfigure}
\usepackage{epstopdf}
\usepackage{dsfont}
\usepackage{graphics,graphicx}
\usepackage{amsmath} 
\usepackage{amssymb}  
\usepackage{psfrag}
\usepackage[all]{xy}
\usepackage{booktabs}
\usepackage{tabularx}
\usepackage{multirow}
\usepackage{pstricks}
\usepackage[boxruled,shortend,linesnumbered]{algorithm2e}
\usepackage{booktabs}

\def\cste{e^}
\def\FTrans{\T}
\def\LieAlg{se(3)}
\def\SOO{SE(3)}
\def\MTwist{\Xi}
\def\ParJoint{\boldsymbol{\sigma}}
\def\JacSkel{{\boldsymbol{\J}}}
\def\RigMot{{\boldsymbol{\rho}}}
\def\tRigMot{{\boldsymbol{\tilde{\rho}}}}

\def\ArtMot{{\boldsymbol{{\omega}}}}

\def\DJoint{{\theta}}		
\def\NPoint{{n}}			
\def\Depl{{\ArtMot}}			
\def\eDepl{{\boldsymbol{\hat{{{\ArtMot}}}}}}	
\def\tDepl{{\boldsymbol{\tilde{{{\ArtMot}}}}}}	
\def\tkerw{{\tilde{\w}}}	
\def\Forc{{\f}}			
\def\tForc{{\tilde{\Forc}}}	

\def\szSupp{{s}}		

\def\MOp{\bf{\Pi}}			
\def\Obs{\boldsymbol{\nu}}			
\def\ndof{{d}}
\def\PseJacSkel{{\K^{+}}}	 
\def\kJacSkel{{\N}}	 	
\def\supf{{\Qc}}			
\def\bsupf{{\bar{\Qc}}}			
\def\stSupp{{\kbracket{\ndof}}}			
\def\supsize{{s}}		

\def\SO3{{\mathbb{SO}(3)}}

\newkfunc{\support}{\mathrm{supp}}
\newkfunc{\acts}{\mathrm{acts}}
\newkfunc{\card}{\mathrm{card}}
\newkmacro\trace[1]{\mathrm{tr}\kparen{#1}}
\newkmacro\frobnorm[1]{\kvvbar{#1}_\textrm{F}}
\newkmacro\frobdot[2]{\kangle{#1,#2}_\textrm{F}}

\DeclareMathOperator{\spann}{span}
\newcolumntype{R}{>{\centering\arraybackslash}X}
\newcolumntype{C}{>{\raggedcenter\arraybackslash}X}
\def\CBSS{\textit{Zhou et al. \cite{ZHOU_SPARSE_2017}}}
\def\CBSV{\textit{Zhou et al. \cite{zhou_sparseness_2016}}}
\def\LfD{\textit{Tome et al. \cite{tome_lifting_2017}}}
\def\SEB{\textit{Martinez et al. \cite{martinez_simple_2017}}}
\def\VNECT{\textit{Mehta et al. \cite{mehta_vnect:_2017}}}

\def\LCR{\textit{Rogez et al. \cite{rogez_lcr-net:_2017}}}
\def\OUR{\textit{Ours}}
\def\Human{\textsc{Human3.6M}}
\def\Panoptic {\textsc{Panoptic}}
\def\MPIINF {\textsc{MPI-I3DHP}}

\title{On the Exact Recovery Conditions of 3D Human Motion\\
		from 2D Landmark Motion with Sparse Articulated Motion}
\date{}
\author{Abed Malti\\OCTI-INC, Los Angeles, USA\\Universit\'e de Tlemcen, Algeria}
\begin{document}

\maketitle

\begin{abstract}
	 In this paper, we address the problem of exact recovery condition in retrieving 3D human motion from 2D landmark motion. We use a skeletal kinematic model to represent the 3D human motion as a vector of angular articulation motion.  
	 We address this problem based on the observation that at high tracking rate, regardless of the global rigid motion, only few angular articulations have non-zero motion. We propose a first ideal formulation with $\ell_0$-norm to minimize the cardinal of non-zero angular articulation motion given an equality constraint on the time-differentiation of the reprojection error. The second relaxed formulation relies on an $\ell_1$-norm to minimize the sum of absolute values of the angular articulation motion. This formulation has the advantage of being able to provide 3D motion even in the under-determined case when twice the number of 2D landmarks is smaller than the number of angular articulations. We define a specific property which is the Projective Kinematic Space Property (PKSP) that takes into account the reprojection constraint and the kinematic model. We prove that for the relaxed formulation we are able to recover the exact 3D human motion from 2D landmarks if and only if the PKSP property is verified. We further demonstrate that solving the relaxed formulation provides the same ground-truth solution as the ideal formulation if and only if the PKSP condition is filled. Results with simulated sparse skeletal angular motion show the ability of the proposed method to recover exact location of angular motion. We provide results on publicly available real data (\Human, \Panoptic~and \MPIINF).
\end{abstract}


\section{Introduction}
Exact recovery condition of 3D human pose in monocular image sequences is an open problem in computer vision. This problem is intrinsically ill-posed since a 2D motion of skeletal landmarks can be explained by multiple 3D skeletal motion and multiple camera poses \cite{agarwal_recovering_2006, park_3d_2011}. The goal of this paper is to  formalize  exact recovery conditions for 3D sparse motion from 2D motion of landmarks which are detected on the human body. Our study assumes a kinematic model with multiple degrees of freedom as depicted in figure \ref{fig:skeleton34}. The sparsity prior is thus applied on the vector of skeletal degrees of freedom. This 3D motion can be either camera motion, body to camera global rigid motion or pure articulated motion. 
\begin{figure}[htbp]
	\begin{center}
		\includegraphics[width=0.4\textwidth]{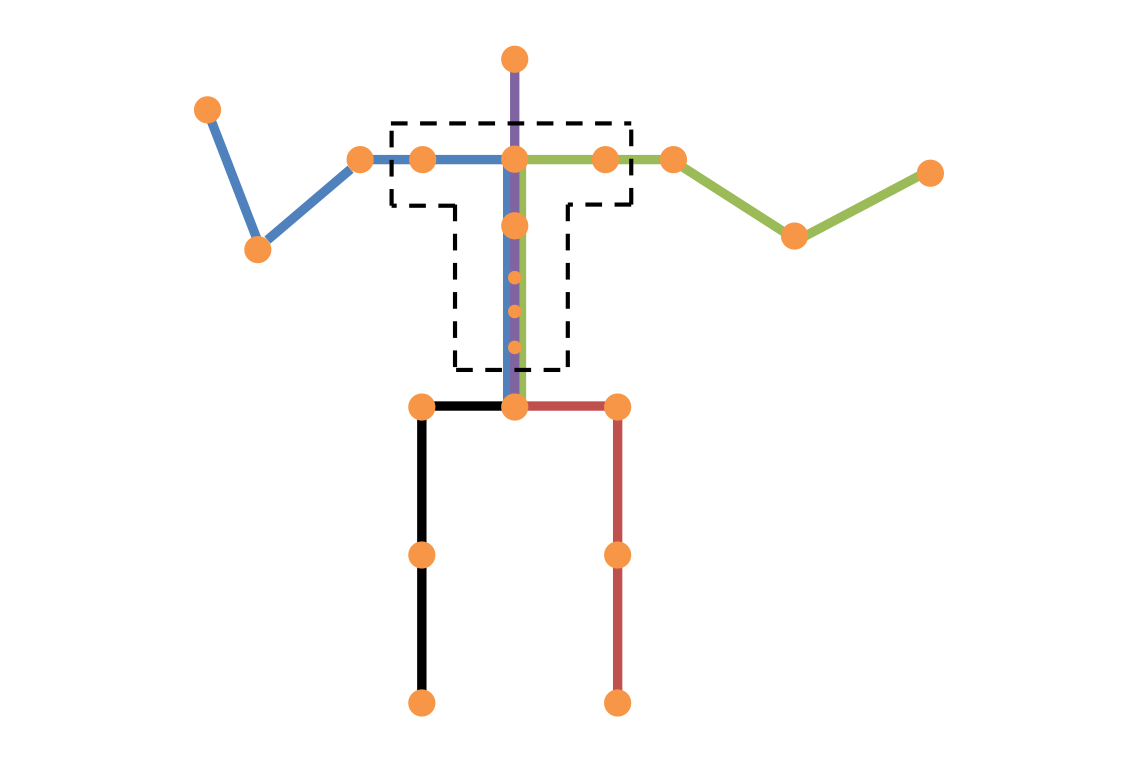}
		\includegraphics[width=0.3\textwidth]{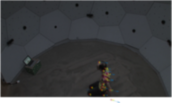}
		\includegraphics[width=0.15\textwidth]{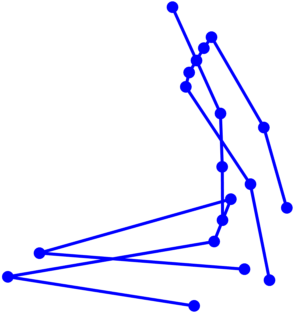}
		\caption{\label{fig:skeleton34}Top: skeleton and joints. Usually, the links enclosed in the T-dashed lines are considered rigid. One advantage of the proposed approach is to be guaranty exact recovery even when considering additional occluded angular joints, for instance in the spine and in the scapula. Bottom: sample reconstruction from \Panoptic~dataset. The sparse articulated motion prior allows us to recover high resolution angular joints with standard detected 2D joints.}
	\end{center}
\end{figure}

3D human from single image has been very productive research area since the last decade \cite{chen_3d_2009}, \cite{ramakrishna_reconstructing_2012},\cite{simo-serra_single_2012},\cite{wang_robust_2014}. Many previous works, \eg \cite{ramakrishna_reconstructing_2012} and \cite{fan_pose_2014}, assume the shape to be retrieved as a linear combination of a finite set of basis shapes. These basis shapes are usually obtained from Principal Component Analysis (PCA). The inference model is built such that the projection of the linear combination of 3D shapes fits the set of corresponding 2D skeleton landmarks in the image. In this case, the problem is to find the coefficients of linear combination and the camera-to-body pose (rotation and translation). The general solving scheme follows an iterative optimization that alternates between optimizing over the coefficients of the linear combination and the camera-to-body relative pose. This standard approach is non-convex and has many failure cases that are not only inherent to the ill-posed aspect but also to bad shape and pose initialization. More recent approach uses a convex relaxation of the formulation to recover a global optimum \cite{ZHOU_SPARSE_2017}. If this approach solved the initialization issue, it does not provide novel cues about the ambiguity problem. These approaches are extended to 3D reconstruction  from image sequence by imposing smoothness over the coefficients of the basis shapes and the camera poses \cite{zhou_sparseness_2016, WANDT_3D_2016}. However, if a sample of frame-wise 3D reconstructions are not exactly recovered from the set of possible poses it will contaminate the smoothing process and the proposed methods do not provide any guaranty of success. Instead of using a low-rank subspace of shapes, methods based on known articulated skeleton have also been investigated \cite{taylor00a}, \cite{guan_estimating_2009}, \cite{leonardos_articulated_2016}, probabilistic graphical models \cite{sigal06a}, \cite{andriluka_2d_2014}, explicit regression \cite{elgammal_inferring_2004}, \cite{agarwal06a}. Most of these methods are based on $\ell_2$ norm minimization and are sensitive to noise in data and skeleton modeling. To limit the noise in skeleton modeling~\cite{theobalt_enhancing_2003} and~\cite{carranza_free-viewpoint_2003} used an initialization step to estimate a template of the skeletal structure through a silhouette-based fitting.\\

Reconstructing 3D human motion from 2D sequence of skeleton landmarks has been first attempted in \cite{bregler_tracking_1998}. The proposed approach used an over constrained least squares method to recover the motion of the articulated angles and the rigid camera-to-body pose. The authors used an optical flow tracking to provide as many skeleton landmarks as possible to over-constrain the reconstruction problem. Even if the approach is convex, there is no guaranty of recovering the exact solution. Ambiguities that arise in 3D human pose reconstruction using an image sequence with kinematic modeling has been addressed by Park et al. \cite{park_3d_2011}. They proposed to reconstruct a 3D articulated trajectory given its 2D projection and the global camera pose at each instant. They applied constant limb length through time as spatial and temporal constraints to smooth the 3D motion. They showed that at every frame, there exist two solutions which satisfy each instantaneous 2D projection and articulation constraint. They demonstrated that resolving this ambiguity by enforcing smoothness is equivalent to solving a binary quadratic programming problem.

In this work, we address the problem of exact recovery conditions of 3D human motion from monocular 2D landmark motion when the articulated motion undergoes a sparse motion. We take advantage from recent state of the art advances in real-time 2D detection of human pose \cite{cao_realtime_2017, insafutdinov_arttrack:_2017} to compute 2D motion landmarks. We use kinematic human model to physically constrain the reconstructed 3D motion. This model is represented by a set of angular articulations connected by limbs (bones) of constant lengths. Each single rotation represents a Degree of Freedom (DoF). In a 3D reconstruction setup, the choice of the kinematic model of the human body is a trade-off  between the computational complexity and the need of the target application in recovering coarse or fine poses. A model with $150$ DoF \cite{yoshimura_multi-body_2005}, \cite{murai_musculoskeletal-see-through_2010} was proposed to study the exposure of human body to vibration in car seats. Less complex models were used with $58$ and $34$ DoFs in \cite{maita_influence_2013} and \cite{ayusawa_identifiability_2014}. The model with $34$ DoFs represents the most popular model that is used in daily activities such as locomotion. In addition to the $34$ DoFs, $6$ DoFs are used to define the rigid transform between the camera frame and the coordinates corresponding to the root base-link. Without loss of generality, this base-link is considered at the center hip. To be able to recover these degrees of freedom, we usually require more observation than DoFs to solve an over-determined least squares system. If we use specific detected landmarks in a least squares setup, we are not able to use kinematic models that contain more than two times the number of available landmarks. In our approach, this is possible since we use a sparse motion prior which states that only few DoFs moves from frame to frame when one have high enough frame rates acquisition. In the proposed formulation, we do not require to impose sparse constraint on the rigid camera to body motion. This flexibility corresponds to realistic context since often the body's base-link or camera motion change the whole rigid DoFs. However, the vector of angular articulations varies sparsely from frame to frame when the acquisition rate is fast enough. 

\noindent\textbf{Contributions and specificities.} In this paper, we study the formalization and well posedness of recoevering 3D articulated motion from 2D landmark motion under the hypothesis of sparse articulated motion.  This approach allows us to recover large number of degrees of freedom with fewer 2D skeleton's landmarks. We propose two formulations for reconstructing 3D human motion from 2D human motion. The first ideal formulation (IF) with $\ell_0$-norm represent the problem as seeking for the vector of angular articulation motion with minimum number of non-zero elements subject to a differential equality constraint on the reprojection error.
The second relaxed convex formulation (RF) uses an $\ell_1$-norm instead of an $\ell_0$-norm. We introduce the Projective Kinematic Space Property (PKSP) that jointly encodes the coupling between rigid motion and articulated motion. In theorem~\ref{theorem2}, we prove that we can recover the exact 3D articulated motion with the exact non-zero elements of the motion vector if and only if the PKSP is filled. It also proves that we are able to recover the exact camera to body rigid pose. We establish in  proposition~\ref{proposition1} under which circumstances the solution of the relaxed formulation (RF) is also the solution of the ideal formulation (IF). 

The proposed approach addresses the camera to body rigid transform but does not provide any cue about decoupling pure rigid camera motion from pure rigid body motion. It uses a perspective projection camera model to recover exact camera to body relative rigid motion. The convex formulation proposed in this paper can potentially serve as a building block or provide a good initialization to improve the performance of the existing methods with additional preprocessing on the 2D detected landmarks \cite{mehta_vnect:_2017}. Lower and upper bounds on angular articulations are physically justified by imposing the articulated rotational angle to stay within its anatomical limits \cite{stoll_fast_2011}.
This paper is inspired from an earlier version~\cite{Mal17} on Shape-from-Templated when the shape is spatially sparsely deformed. Similar projective property is instantiated for the case of articulated human body. We additionally provide real experimental validation using publicly available dataset.

\section{Related Work}
The closest related work includes physics-based methods that use the 2D detected skeleton joints in a reprojection error constraint. According to the set of used prior, state-of-the-art approaches can be split in several subsections. Some methods can be part of multiple subsections but we only mention them in one paragraph. We focus our related work study on monocular methods with both single image and image sequence approaches. We can also split the state-of-the-art methods to physics-based and deep learning approaches.\\
\noindent\textbf{Lower dimensional pose subspaces. }Ramakrishna et al.~\cite{ramakrishna_reconstructing_2012} propose a sparse linear representation in an overcomplete dictionary. They use a matching pursuit algorithm to minimize the reprojection error under anthropometric regularization. Fan et al. \cite{fan_pose_2014} propose to improve the performance of \cite{ramakrishna_reconstructing_2012} by enforcing locality when building the pose dictionary. Akhter et al. \cite{akhter_pose-conditioned_2015} add limb-angle constraint into the sparse representation to reduce the possibility of invalid reconstruction. Wang et al. \cite{wang_robust_2014} use a 2D human pose detector to automatically locate the joints and integrate a robust estimator to handle inaccurate joint locations. They constrained the proportion of $8$ selected limbs to be constant. Zhou et al. \cite{zhou_spatio-temporal_2014} address the human pose estimation problem as a matching problem in which a learned spatio-temporal pose model is matched to point trajectories extracted from a video. Zhou et al. \cite{ZHOU_SPARSE_2017} propose a convex relaxation with $\ell_1$-norm on the coefficient of the used basis shapes. Zhou et al. \cite{zhou_sparseness_2016} extend the former approach to account for temporal smoothness by imposing regularization terms on the coefficients and on the camera to human pose orientation in an image sequence. Wandt et al. \cite{WANDT_3D_2016} impose a temporal bone lengths constancy in an image sequence  within both periodic and non-periodic motion setting.\\
\noindent\textbf{Factorization in image sequence. }Unlike lower dimensional pose subspaces approaches, in factorization approaches the set of basis shape dictionary is supposed to be unknown and is part of the problem.   In the seminal work by Tomasi and Kanade \cite{tomasi92a} the 3D motion to reconstruct is supposed to be piece-wise rigid. Bergler et al. \cite{bregler00a} extend the former approach to deforming shapes. The reconstructed 3D shape is expressed as a linear combination of static basis shapes. Xiao et al. \cite{xiao04a} exhibit the fact that using only orthonormal constraint on the camera rotation leads to multiple solutions. They add constraints to the basis shape to retrieve a unique closed-form solution. Some authors as \cite{torresani03a, torresani08a} and \cite{torresani01a} impose a Gaussian constraint on the linear coefficients to avoid non-rigid self-calibration. Later on, Akhter et al. \cite{akhter_trajectory_2011} show that the solution in \cite{xiao04a} is still ambiguous. They add a strong prior which is to constrain Non-rigidity by a periodic base function. Zhu et al. \cite{zhu_3d_2011} propose to use small keyframes to avoid ambiguities between point and camera motion. They also use $\ell_0$ minimization to enforce a sparsity constraint on the trajectory basis coefficients. Dai et al. \cite{dai12} used $\ell_1$-minimization to limit the number of non-zero linear coefficients. Paladini et al.~\cite{paladini_optimal_2012} formalize the metric upgrade of the solution from affine to Euclidean space with a convex relaxation approach. Del Bue et al.~\cite{bue_bilinear_2012} propose a manifold constrain approach to solve bilinear factorization problems in the case of missing data measurement. Agudo et al.~\cite{agudo_good_2014} formalize solving the basis shape as an eigenvalue problem and propose to retrieve the coefficients and camera pose in a sliding window bundle adjustment framework.\\
\noindent\textbf{Skeleton/anthropometric. } Bergler and Malik~\cite{bregler_tracking_1998} propose an intensity-based tracker using a kinematic modeling with twists and exponential maps. Sminchisescu et al.~\cite{sminchisescu_kinematic_2003} propose a kinematic body tracker by enforcing online ambiguity rejection. Sminchisescu et al.~\cite{sminchisescu_variational_2004} propose a body tracker using a mixture density smoothing model in a bayesian framework. Hasler et al.~\cite{hasler_multilinear_2010} constrain bone length constancy in image sequence. Park et al.~\cite{park_3d_2011} propose to reconstruct a 3D articulated trajectory of a joint given the trajectories of: {\em{(1)}}  its two dimensional projection, {\em{(2)}} the 3D parent joint, and {\em{(3)}} the camera pose. The reconstruction ambiguity was resolved by enforcing trajectory smoothness and solving a binary quadratic programming problem.\\
\noindent\textbf{Skeleton with deep learning. }Tome et al.~\cite{tome_lifting_2017} and Martinez et al.~\cite{martinez_simple_2017}  build a network that given 2D joint locations predicts 3d positions. Many authors~\cite{mehta_monocular_2016,rogez_lcr-net:_2017,pavlakos_coarse--fine_2017} and \cite{zhou_towards_2017} propose to directly regress 3D human pose joint location from an image. Moreno-Noguer \cite{moreno-noguer_3d_2017} formulates the problem as a 2D-to-3D distance matrix regression. Katircioglu et al. \cite{katircioglu_learning_2018} learn a high-dimensional latent pose representation and accounts for joint dependencies via an auto-encoder framework. They further propose an efficient Long-Short-Term Memory network to enforce temporal consistency on 3D pose predictions. Zhou et al. \cite{zhou_deep_2016} embed a kinematic model into a deep neural network learning for general articulated object pose estimation. The kinematic function that is used is differentiable so that it can be used in the gradient descent based optimization in network training. Mehta et al. \cite{mehta_vnect:_2017} use the same pose regressor for both 2D and 3D joint location in a real-time setting. To ensure stability and temporal consistency, bone lengths and bone direction fitting are used. \\
\noindent\textbf{Dense 3D human with deep learning. }Bogo et al.~\cite{bogo_keep_2016} propose the first method to estimate the 3D pose as well as the 3D surface mesh of the human body from a single unconstrained image. Alldieck et al.~\cite{alldieck_optical_2017}  exploit properties of optical flow to temporally constrain the reconstructed human motion. They minimize the difference between computed flow fields and the output of a flow renderer. Kanazawa et al.~\cite{kanazawa_end--end_2017} propose an end-to-end framework for reconstructing a full 3D mesh of a human body from a single RGB image. Xu et al. \cite{xu_monoperfcap:_2017} reconstruct articulated human skeleton motion as well as medium-scale non-rigid surface deformations in general scenes. They use a low dimensional trajectory subspace to resolve the ambiguities of the monocular reconstruction problem. Generally, methods with deep learning approaches reconstruct relative depths but the reprojection of the 3D skeleton do not accurately match the 2D location of the body's landmark in the image. Furthermore, they use a normalized image size which breaks the perspective model. The loss function is based on this 2D normalized inputs.\\
\noindent\textbf{Notations. }
Normal letters        = scalars, \eg $a,b,A,B, etc$.
Bold small letters   = vectors, \eg $\a,\b, etc$.
Bold capital letters = matrices, \eg $\A,\B, etc$.
Calligraphic letters = sets, \eg $\Ac,\Bc, etc$.
$i$th element of vector $\a$: $a_i$.
Element located at row $i$ and column $j$ of $\A$: $A(i,j)$.
$\ktranspose{\a}$ denotes the transpose of $\a$.
$\bar{\Ac}$ denotes the complementary set  of $\Ac$.
If $\supf$ is a set of indices, $\A_\supf$ represents the submatrix of $\A$ made up of the columns indexed by $\supf$. 
$\a_\supf$ represent either the restriction of $\a$ to the indices in $\supf$, or the vector which coincides  with $\a$ on the indices in $\supf$ and is extended to zero outside $\supf$. It should be clear from the context which notation is meant. 
$\stSupp=\{1,\ldots,3n \}$, $n\in\Nbb$.
$\stSupp^s$: All subsets of $\stSupp$ of cardinal $\supsize$.
$\supp[\a]=\kbrace{i : a(i)\neq 0}$ is
the set of integers indexing the non-zero elements of $\a$.
$\kbracket{\a_i}_{i=1}^k=\kbracket{\a_1 \ldots \a_k}$.
$\kbrace{\a_i}_{i=1}^k=\kbrace{\a_1 \ldots \a_k}$.
$\ker\kparen{\A} = \kbrace{\x : \A\x=\0}$.
$\spann\kparen{\A} = \kbrace{\y : \A\x=\y, \x \text{ real vector}}$.
$\kvvbar{\cdot}_0$ is the $\ell_0$-norm.
\begin{figure}[htbp]
	\begin{center}
		\includegraphics[width=0.45\textwidth]{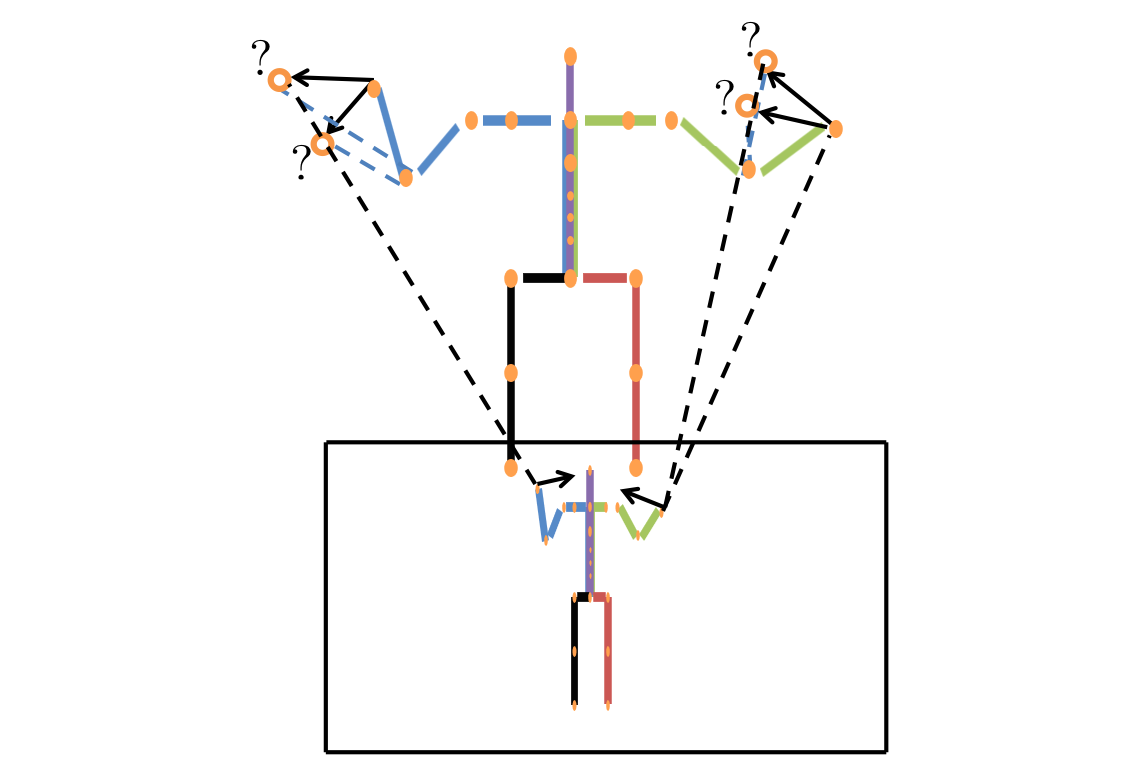}
		\caption{\label{fig:skeleton34_2}Sketch of possible ambiguities when reconstructing 3D differential motion from 2D differential motion.}
	\end{center}
\end{figure}	
\section{Formulation Tools}
\subsection{The Kinematic Model}
In this section, we describe the main steps to obtain a linear model mapping differential rigid and articulated motion to differential image motion. This kinematic modeling is inspired from \cite{bregler_tracking_1998} which used twists and exponential maps to efficiently map angular joints speed to body's 3D point velocities. The body's point velocities are projected with a perspective model to obtain 2D image feature velocities. Without loss of generality, we assume the base root of the human skeleton to be located at the center of the hip. This hypothesis is standard in many works \cite{mehta_monocular_2016,tome_lifting_2017,mehta_vnect:_2017}. It is well suited for rigging and skinning applications. 
In this work, the differential rigid motion stands only for relative camera to skeleton motion. Differential articulated motion stands for pure angular joints displacement. As will be seen, the proposed formalism do not separate absolute camera motion from global skeleton rigid motion. 
\subsubsection{Pure Differential Articulated Motion}
The human skeleton can be seen as five serial kinematic chains all connected to the same basis coordinate frame located at the hip. The five chains are hip to head, hip to right wrist, hip to left wrist, hip to right ankle and hip to left ankle as depicted in figure \ref{fig:skeleton34}. 
The skeleton chain can be seen as a tree where the spine is a common tree trunk to the upper body. 
Let us consider $\p_i=\ktranspose{\kbracket{x\ y\ z}}\in\Rbb^3$, $i\in\kbracket{\NPoint}$, as a skeleton point positioned at some limb. $\NPoint$ being the number of considered skeleton's point. A set of parent joints for a given skeleton's point is the set of all joints contained in the serial kinematic chain linking the hip to the current point (obviously the hip is a parent joint to all skeleton's joints). 
Each skeleton's joint is a rotational joint which is either one-rotation degree of freedom (elbows and knees) or three-rotation degrees of freedom (shoulders, left-right hips, spine joints, neck, etc).
Let us denote $\kbracket{\theta_1\,\ldots\,\theta_d}$ the vector of the whole articulated angular degrees of freedom of the 
kinematic structure. $\ndof$ being the total number of pure articulated angular degrees of freedom.
We denote $\ParJoint(i)=\kbrace{\sigma_1(i)\ldots \sigma_m(i)}$ the set of indices of the parent angular articulations of point $i$, 
$i\in\kbracket{\NPoint}$. $m\in\kbracket{\ndof}$ is the number of parent angular degrees of freedom. 
The twist $\MTwist\in\LieAlg$ constructed from the unit axis of rotation and translation allows us to obtain the rotation transform thanks to the exponential $\cste{\MTwist\, \DJoint}\in\SOO$ for $\DJoint\in\Rbb$ \cite{murray_mathematical_1994}. Where $\LieAlg$ is the Lie Algebra of the group of rigid transforms $\SOO$.
Let us denote $\FTrans_c\in\SOO$ as the camera to hip transform.
Let us consider a 3D skeleton's point with constant location $\p_i^0$ in the reference frame attached to $\sigma_m(i)$. The location $\p_i$ of this point in the camera frame is given as 
\begin{align}
	\p_i = \FTrans_i\,\p_i^0,
\end{align}
where $\FTrans_i\in\SOO$ is the rigid transform from the camera to the coordinate frame that is attached to $\sigma_m(i)$ and is given by
\begin{align}
	\FTrans_i=\FTrans_c\cste{\MTwist_{\ParJoint_1(i)}\, \DJoint_{\ParJoint_1(i)}}\ldots \cste{\MTwist_{\ParJoint_m(i)}\, \DJoint_{\ParJoint_m(i)}},\ i\in\kbracket{\NPoint}. 
\end{align}

The differential 3D displacement of point $\p_i$ due to pure differential angular articulation $\dot{\theta}_j$, in the camera reference frame, is computed as  
\begin{align}
	\dot{\p}_i&=\J_{j}\kparen{\p_i}\,\dot{\theta}_j,
\end{align}
where
\small{
	\begin{align}
		\J_{j}(\p_i) =\left\{\begin{array}{c}
			\0, \text{ if } j\notin \sigma(i).\\
			\kMRot\kparen{\pderiv{\T_i}{\theta_j}\kinv{\T_i}}\,\p_i+\kMTrans\kparen{\pderiv{\T_i}{\theta_j}\kinv{\T_i}},\text{ else}.
		\end{array}
		\right.
	\end{align}
}
Given
\begin{align}
	\label{eq:twist}
	\pderiv{\T_i}{\theta_{\ParJoint_k(i)}}\kinv{\T_i} = \FTrans_{\ParJoint_{k-1}(i)}\,\MTwist_{\ParJoint_k(i)}\,\kinv{\FTrans}_{\ParJoint_{k-1}(i)},\ i\in\kbracket{\NPoint}.
\end{align}
$\kMRot$ and $\kMTrans$ operators extract respectively rotation and translation parts of an $\SOO$ matrix.  Let us denote by $\ArtMot=\ktranspose{\kbracket{\dot{\theta}_1\ldots\dot{\theta}_\ndof}}$ the concatenated vector of differential articulated motion. We can thus write the differential 3D displacement with respect to all parent angular articulations as follows
\begin{align}
	\dot{\p}_i=\kbracket{\J_1(\p_i)\,\ldots\,\J_\ndof(\p_i)}\,\ArtMot.
\end{align}
\subsubsection{Pure Differential Rigid Motion}
Let us denote by $\RigMot=\ktranspose{\kbracket{\ktranspose{\mathbf{v}}\,\ktranspose{\w}}}$ the concatenated vector of differential rigid motion. The differential displacement of a point $\p_i$ on the skeleton with respect to a global camera to body differential rigid motion, in the camera reference frame, is computed as
\begin{align}\dot{\p}&=\kbracket{\begin{array}{cc}\I& \kbracket{\begin{array}{ccc}0 &p_z &-p_y\\-p_z&0&p_x\\p_y&-p_x&0 \end{array}}\end{array}}\RigMot\\&=\kbracket{\I\, \hat{\p}}\RigMot\\&=\G(\p)\RigMot.
\end{align}
$\hat{\p}$ denotes the skew symmetric matrix associated to $\p$. $\mathbf{v}\in\Rbb^3$ and $\w\in\Rbb^3$ are respectively the amounts of differential translation and differential rotation of the rigid motion. 
\subsubsection{Combined Differential Motion}
The differntial displacement of any 3D skeleton's point $\p_i$ that is caused by both differential rigid and articulated motions can be written as
\begin{align}
	\dot{\p}_i=\kbracket{\G(\p_i)\ 	\J_1(\p_i)\,\ldots\,\J_\ndof(\p_i)}\ktranspose{\kbracket{\ktranspose{\RigMot}\,\ktranspose{\ArtMot}}}.
\end{align}
For a given set of 3D skeleton's points $\p_1,\ldots, \p_\NPoint$, the former relation extends to
\begin{align}
	\begin{bmatrix}
		\dot{\p}_1\\\vdots\\\dot{\p}_\NPoint
	\end{bmatrix}
	=
	\bf{\boldsymbol{\Gamma}}\,\RigMot+\JacSkel\,\ArtMot
\end{align}
where 
\begin{align}
	\boldsymbol{\Gamma}=\ktranspose{\kbracket{\ktranspose{\G(\p_1)}\ldots\ktranspose{\G(\p_\NPoint)}}},
\end{align}
and
\begin{align}
	\JacSkel=\begin{bmatrix} \J_1(\p_1)&\ldots&\J_\ndof(\p_1)\\\vdots&\ddots&\vdots\\\J_1(\p_\NPoint)&\ldots&\J_\ndof(\p_\NPoint)\end{bmatrix}.
\end{align}
$\boldsymbol{\Gamma}$ and $\JacSkel$ are respectively $3\NPoint\times 6$ and $3\NPoint\times\ndof$ matrices. For notation convenience, we do not specify that they are function of the argument tuple $(\p_1,\ldots,\p_\NPoint)$. The matrix $\JacSkel$ is the Jacobian of the kinematic skeleton structure and is non-dense. For a given point, the non-zero row-wise elements correspond to the parent angular degrees of freedom. $\boldsymbol{\Gamma}$ is a dense matrix since any differential rigid motion modifies the location of any skeleton's point.  
\begin{lemma}
	\begin{enumerate}
		\item If $\NPoint=1$, then the null space of $\boldsymbol{\Gamma}$ is straightforward: 
		\begin{align}
			\ker(\boldsymbol{\Gamma}) = \spann\kparen{\ktranspose{\kbracket{-\ktranspose{\hat{\p}}\, \ktranspose{\I}}}}.
		\end{align}	
		\item If $\NPoint=2$, let us consider $\p_1=\ktranspose{\kbracket{x_1\,y_1\,z_1}}$ and $\p_2=\ktranspose{\kbracket{x_2\,y_2\,z_2}}$ with $z_1\neq z_2$, then
		\begin{align}
			\ker\kparen{\boldsymbol{\Gamma}} = \spann\kparen{\frac{1}{z_2-z_1}\ktranspose{\kbracket{\ktranspose{-\kparen{\p_2\times\p_1}}\ \ktranspose{\kparen{\p_2-\p_1}}}}}.
		\end{align}
		\item If $\NPoint\geq 3$ such that there are non-aligned three skeleton's points, then the null space is trivial
		\begin{align}
			\ker\kparen{\boldsymbol{\Gamma}} = \emptyset.
		\end{align}
	\end{enumerate}
\end{lemma}
\begin{proof}
Point 1) and 2) of the lemma can be demonstrated with simple mathematical development or with a symbolic calculs software as \textsc{Matlab}.
\end{proof}	
The proof of the lemma is straightforward and is not developed in this paper. It can be easily checked with any symbolic calculus software like Matlab. In our study, we consider that we have access to at least three non-aligned skeleton's point.


\subsection{The Differential Reprojection Constraint}
In this paper, we assume that we are able to detect and track $\NPoint$ sparse points in the skeleton where usually $2\NPoint<\ndof$.
Let us consider $\u(\p)=\ktranspose{\kparen{\frac{x}{z},\frac{y}{z}}}$ as the perspective projection of a 3D skeleton's point $\p=\ktranspose{\kparen{x,y,z}}$ (without taking into account camera intrinsics).
The differential motion of the 2D projection caused by the differential motion of the 3D point can be written as the following linear mapping 
\begin{align}
	\dot{\u}(\p) &= \kbracket{\begin{array}{ccc} \frac{1}{z}&0&-\frac{x}{z^2}\\0&\frac{1}{z}&-\frac{y}{z^2}\\\end{array}} \dot{\p}\\
	&=\M(\p)\,\dot{\p}.
\end{align}
For a given set of $\NPoint$ skeleton's points

\begin{align}
	\label{eq:motion1}
	\begin{bmatrix}
		\dot{\u}(\p_1)\\\vdots\\\dot{\u}(\p_\NPoint)
	\end{bmatrix}
	=
	\MOp\,\kbracket{
		\bf{\boldsymbol{\Gamma}}\quad\JacSkel}
	\,\ktranspose{\kbracket{\ktranspose{\RigMot}\,\ktranspose{\Depl}}},
\end{align}
where $\MOp$ is a bloc diagonal matrix 
\begin{align}
	\MOp=\begin{bmatrix} \M(\p_1)&&\\&\ddots&\\&&\M(\p_\NPoint)\\\end{bmatrix}.
\end{align}
For convenience, we rewrite \ref{eq:motion1} and adopt the following contracted expression
\begin{align}
	\Obs=\MOp\,
	\bf{\boldsymbol{\Gamma}}\,\RigMot+\MOp\,\JacSkel
	\,\Depl,
\end{align}
where $\Obs$ is the vector of all observed 2D differential displacements.
\begin{lemma}
	Let us consider $\kbrace{\e_i}_{i=1}^\NPoint$ as denoting the canonical orthonormal basis of $\Rbb^\NPoint$. For $\NPoint$ sekeleton points, the null space of $\MOp$ is of dimension $\NPoint$ and is defined as
	\begin{align}
		\ker\kparen{\MOp}=\spann\kparen{\kbracket{\e_i\otimes\p_i}_{i=1}^{\NPoint}}.
	\end{align}
\end{lemma}
Where $\otimes$ denotes the Kronecker product. This null space formalizes the fact that sliding the skeleton's points along the sightlines does not change their projection on the image plane. As will be seen, in studying the exact recovery condition it is important to know if there exists some rigid differential motion of given skeleton's point that does not change their reprojection onto the image. This statement is formalized by the following proposition.
\begin{proposition}
	If $\NPoint\geq 3$ such that there are no skeleton's points aligned, then the null space $\ker\kparen{\MOp\Gamma}$ is trivial.
\end{proposition}
The proof can be understood by simply seeing that there is no differential rigid motion of three points or more that keeps the same reprojection. This question is nothing but the PnP problem where it was stated in \cite{haralick94} that for $\NPoint> 3$, there is no rigid transformation  that can give similar perspective projection. When $\NPoint=3$, the null space is still trivial since the differential motion allows us to discard ambiguities due to wide rigid motion. For $\NPoint\leq 2$, the null space is non-trivial and this case is not addressed because it is unusual and not interesting in practice. The proof of the lemma is straightforward and is not developed in this paper.	It can be easily checked with any symbolic calculus software like Matlab. In our study, we consider that we have access to at least three non-aligned skeleton's point.
%
%
\subsection{Support and Sparsity of the Angular Articulated Motion}
\begin{definition}
	We call the support of $\Depl$, the location of its non-zero elements
	\begin{align}
		\supf\triangleq&\{i |\,  \dot{\theta}_i\neq0,\ i\in\kbracket{\ndof}\}, 
	\end{align} 
	The size of the support being the number of non-zero elements of $\Depl$
	\begin{align}
		\kvbar{\supf}\triangleq\kvvbar{\Depl}_0.
	\end{align} 
\end{definition}
We denote by $\bar{\supf}$ the complement of $\supf$ in $\kbracket{\ndof}$. In the case of sparse articulated motion, the cardinal $|\supf|$ is expected to be far smaller than $|\bar{\supf}|$. 
\section{Problem Formulation}

Given a vector $\Obs\in\Rbb^{2\NPoint}$ of 2D differential motion of $\NPoint$ skeleton's points, the goal is to study the possiblities of exactly recovering 
$(\RigMot,\Depl)\in\Rbb^{6}\times\Rbb^{\ndof}$ from $\Obs=\MOp\,\bf{\boldsymbol{\Gamma}}\,\RigMot+\MOp\,\JacSkel\,\Depl$.


\subsection{Ideal Formulation (IF)}
Under the assumption of sparse articulated motion it is natural to minimize the number of non-zero elements of the differential articulated motion $\ArtMot$ under the differential reprojection constraint
\begin{align}
	\label{eq:S1-idealProb}
	\tag{IF}
	\eDepl \in &\kargmin_{\tDepl} \kvvbar{\tDepl}_0\\\nonumber
	\text{s.t. }&\Obs=\MOp\,\bf{\boldsymbol{\Gamma}}\tRigMot+\MOp\,\JacSkel\,\tDepl.
\end{align}
Where $\Depl$ refers to the (unknown) ground truth; $\eDepl$ is an estimate of $\Depl$ computed from $\Obs$; $\tDepl$ is a ``trial'' vector used to define the optimization problem. The same notations apply for $\RigMot$.

This $\ell_0$-norm formulation even ideal is non convex and the problem to solve is NP-hard \cite{foucart13}. To be able to study and solve this ideal formulation we use a convex relaxtion with $\ell_1$-norm instead. If the solution to the relaxed problem corresponds to the exact motion thus it is the same solution that solves  the ideal problem \ref{eq:S1-idealProb}.	
\subsection{Relaxed formulation (RF)}
The relaxed formulation of the \eqref{eq:S1-idealProb} can be formulated as
\begin{align}
	\label{eq:S1-RelaxedProb}
	\tag{RF}
	\eDepl \in &\kargmin_{\tDepl} \kvvbar{\tDepl}_1\\\nonumber
	\text{s.t. }&\Obs=\MOp\,\bf{\boldsymbol{\Gamma}}\tRigMot+\MOp\,\JacSkel\,\tDepl,
\end{align}

This is a convex problem since it is an $\ell_1$-minimization (convex function) with linear convex constraint. 
It has a unique solution if the feasible set is non-empty.
Further than that, we are interested in establishing specific conditions where the solution corresponds to the ground-truth articulated and rigid motion.
Deriving such conditions allows us to specify when the solution of \eqref{eq:S1-idealProb} corresponds to the solution of \eqref{eq:S1-RelaxedProb}.

\section{\label{sec-Recov-Cond}Exact Recovery Condition for (RF)}
In order to establish the exact recovery condition for \eqref{eq:S1-RelaxedProb}, we first derive an extension of the Null Space Property that we dubbed here the Projective Kinematic Space Property (PKSP). As it is demonstrated, this condition is necessary and sufficient to recover a motion vector $\Depl$ based on the sparsity prior and on its image $\Obs=\MOp\,\bf{\boldsymbol{\Gamma}}\RigMot+\MOp\,\JacSkel\,\Depl$. This property is derived from the seminal work \cite{bandeira13} and which was first derived for Shape from Elastic Template problem in \cite{Mal17}.
\subsection{The Projective Kinematic Space Property (PKSP)}
\begin{definition}
	We say that the triplet $\kparen{\MOp, \boldsymbol{\Gamma}, \JacSkel}$ satisifies the Projective Kinematic Space Property (PKSP) relative to $\supf\subset\stSupp$ if for every $\Depl\in\Rbb^\ndof\backslash\kbrace{0}$ such that $\MOp\JacSkel\,\Depl\in\spann\kparen{\MOp\boldsymbol{\Gamma}}$ and $\ker\kparen{\MOp\boldsymbol{\Gamma}}=0$, we have
	\begin{align}
		\label{eq:PKSP}
		\kvvbar{\Depl_\supf}_1<\kvvbar{\Depl_\bsupf}_1.
	\end{align}
\end{definition}	
This definition states that the PKSP is verified for a given support $\supf$ if for all the non-zero differential articulated motion that project exactly as differential rigid motion, the $\ell_1$ magnitude of differential articulated motion supported on $\supf$ is strictly dominated by the $\ell_1$ magnitude of differential articulated motion supported on the complement $\bsupf$. This definition can be extended to every support of fixed size as follows:
\begin{definition}
	\label{def:pksp-s}
	We say that the triplet $\kparen{\MOp, \boldsymbol{\Gamma}, \JacSkel}$ satisifies the PKSP of order $\supsize$ if it satisfies the PKSP relative to $\supf\subset\stSupp$ for every $\supf\in\stSupp^\supsize$.
\end{definition}

\begin{theorem}
	\label{thm:pksp-q}
	Every differential motion $\kparen{\RigMot,\Depl}\in\Rbb^6\times\Rbb^\ndof$, where the differential articulated motion $\Depl$, is supported on a set $\supf\subset\stSupp$	is the unique solution of \eqref{eq:S1-RelaxedProb} with the observed 2D motion $\Obs=\MOp\,\bf{\boldsymbol{\Gamma}}\RigMot+\MOp\,\JacSkel\,\Depl$ if and only if the triplet $\kparen{\MOp,\boldsymbol{\Gamma},\JacSkel}$ statisfies the PKSP relative to $\supf$. 
\end{theorem}
It is then necessary and sufficient to fit the PKSP condition for any support size to be able to recover exactly the differential motion from the differential image motion.	In other words, if the ambiguous situations between projected differential rigid motion and projected differential articulated motion are such that the absolute magnitude of $\supsize$ elements from differential articulated motion  is strictly dominated by the complement amount of magnitude then any $\supsize$-sparse differential articulated motion can be exactly recoverd with \eqref{eq:S1-RelaxedProb}.
\begin{proof}
	We demonstrate the theorem by the contrapositive of each implication. First let us proof that non-PKSP relative to  a given $\supf\subset\stSupp$ implies non-exact recovery. Let us consider the existence of a given $\Depl\in\Rbb^\ndof\backslash\kbrace{0}$ such that $\MOp\JacSkel\,\Depl\in\spann\kparen{\MOp\boldsymbol{\Gamma}}$ and $\kvvbar{\Depl_\supf}_1\geq\kvvbar{\Depl_\bsupf}_1$. Let $\x=\Depl_\supf$ and $\bar{\x}=-\Depl_\bsupf$, then there exists $\z$ and $\bar{\z}$ in $\Rbb^6$ such that $\MOp\JacSkel(\x-\bar{\x})=\MOp\boldsymbol{\Gamma}(\z-\bar{\z})$. We arrange the terms so that we obtain $\MOp\boldsymbol{\Gamma}\bar{\z}+\MOp\JacSkel\bar{\x}=\MOp\boldsymbol{\Gamma}\z+\MOp\JacSkel\x$. It appears that $\x$, which is supported on $\supf$, is not the minimizer that we are seeking from the 2D differential motion $\Obs=\MOp\boldsymbol{\Gamma}\z+\MOp\JacSkel\x$. Indeed, $\bar{\x}$ is the minimizer since $\kvvbar{\x}_1\geq\kvvbar{\bar{\x}}_1$. Second, let us proof that for a given support $\supf$, non-exact recovery implies non-PKSP. Let us assume $(\z,\x)\in\Rbb^6\times\Rbb^\ndof$ such that $\x$ is supported on $\supf$. Let us consider $(\bar{\z},\bar{\x})\in\Rbb^6\times\Rbb^\ndof$ such that $\bar{\x}$ is not supported on $\supf$. According to the contrapositive assumption we can have  $\MOp\JacSkel\bar{\x}+\MOp\boldsymbol{\Gamma}\bar{\z}=\MOp\JacSkel\x+\MOp\boldsymbol{\Gamma}\z$ with $\kvvbar{\x}_1>\kvvbar{\bar{\x}}_1$. Let us denote $\w=\x-\bar{\x}$ which gives $\MOp\JacSkel \w\in\spann(\MOp\boldsymbol{\Gamma})$. In order to demonstrate that the triplet $\kparen{\MOp,\boldsymbol{\Gamma},\JacSkel}$ does not fill the PKSP condition we have to demonstrate that $\kvvbar{\w_\supf}_1\geq\kvvbar{\w_\bsupf}_1$ which is obtained as follows
	\begin{align}
		\kvvbar{\w_\bsupf}_1&=\kvvbar{\bar{\x}_\bsupf}_1,\ \textrm{since $\x_\bsupf=\0$},\nonumber\\
		&=\kvvbar{\bar{\x}_\bsupf}_1+\kvvbar{\bar{\x}_\supf-\bar{\x}_\supf+\x}_1-\kvvbar{\x}_1\nonumber\\
		&\leq\kvvbar{\bar{\x}_\bsupf}_1+\kvvbar{\bar{\x}_\supf}_1+\kvvbar{\x-\bar{\x}_\supf}_1-\kvvbar{\x}_1\nonumber\\
		&\leq\kvvbar{\bar{\x}}_1-\kvvbar{\x}_1+\kvvbar{\w_\supf}_1\nonumber\\
		&\leq\kvvbar{\w_\supf}_1,\ \textrm{since $\kvvbar{\bar{\x}}_1-\kvvbar{\x}_1\leq 0$},
	\end{align}
	if we recover the exact differential kinematic motion $\Depl$, then the exact differential rigid motion $\rho$ is recovered under the necessary and sufficient condition of $\ker\kparen{\MOp\boldsymbol{\Gamma}}=0$. Thus, it can be exactly determined by solving the equation $\MOp\boldsymbol{\Gamma}\,\rho=\Obs-\MOp\JacSkel\,\Depl$.
\end{proof}	
This exact recovery condition relative to a support $\supf$ can be extended to any support of a given size $\supsize$.
\begin{theorem}
	\label{theorem2}
	Every differential motion $\kparen{\RigMot,\Depl}\in\Rbb^6\times\Rbb^\ndof$, where the differential DoFs motion $\Depl$ is an $\supsize$-sparse vector, is the unique solution of \eqref{eq:S1-RelaxedProb} with the observed image motion $\Obs=\MOp\,\bf{\boldsymbol{\Gamma}}\RigMot+\MOp\,\JacSkel\,\Depl$ if and only if the triplet $\kparen{\MOp,\boldsymbol{\Gamma},\JacSkel}$ statisfies the PKSP of order $\supsize$. 
\end{theorem}
\begin{proof}
	According to definition \ref{def:pksp-s}, for every $\supf\subset\stSupp^\supsize$, $\kparen{\MOp,\boldsymbol{\Gamma},\JacSkel}$ satisfies the PKSP relative to $\supf$. By theorem \ref{thm:pksp-q}, for every such support $\supf$, every differential motion $\kparen{\RigMot,\Depl}$, where $\Depl$ is supported on $\supf$, is exactly recovered from \eqref{eq:S1-RelaxedProb}.
\end{proof}
\subsection{Exact Recovery for (IF)\label{sec:RF-2-IF}}
\begin{proposition}
	\label{proposition1}
	Under the condition of theorem \ref{theorem2} every exact differential motion that is recovered via problem (RF) is also the unique exact solution of problem (IF). 
\end{proposition}
\begin{proof}
	Under the condition of theorem \ref{theorem2}, the recovered solution corresponds to the exact differential motion 
	$\kparen{\RigMot,\Depl}\in\Rbb^6\times\Rbb^\ndof$, where $\Depl$ is $\supsize$-sparse. By hypothesis, this differential articulated motion is exact, explains the observation and satisfies the differential kinematic body model. 
	Moreover, if $\kparen{\boldsymbol{\alpha},\boldsymbol{\beta}}\in\Rbb^6\times\Rbb^\ndof$ is the minimizer of (IF) then $\kvvbar{\boldsymbol{\beta}}_0\leq\kvvbar{\Depl}_0$ so that $\boldsymbol{\beta}$ is also $
	\supsize$-sparse. Since every $\szSupp$-sparse vector is the unique solution of (RF) with the same $\Obs=\MOp\kparen{\PseJacSkel\tForc +\kJacSkel\tkerw}$, it follows that $\boldsymbol{\beta}=\Depl$. The equality $\alpha=\RigMot$ follows from the exactness of differential articulated motion and the necessary and sufficient condition $\ker\kparen{\MOp\boldsymbol{\Gamma}}=0$.
\end{proof}		

We assume that we have one point per link limb. Which means, we have for the left/right sides: hips, shoulders, elbows, wrists, knees and ankles.
If one point is missing then the question that arises is whether or not its direct parent joint can be recovered. 

\section{Implementation}
\noindent\textbf{Algorithmic. }The method implementation and evaluation used \textsc{Matlab R2017a} running upon a \textsc{MAC Pro} (\textsc{OS X $10.13.4$}) with \textsc {Intel Core i5} running at $2\times2.66$ GHz with $16$ GB 1600 MHz DDR3 memory. The resolution of the constrained $\ell_1$ minimization part of (RF) was implemented using ADMM algorithm~\cite{wen10}. The algorithm runs at 200~fps in average upon the described hardware and software configuration. With this computational speed, the proposed approach can be used with additional pre-processing steps to stabilize and correct the 2D detected landmark detection as a whole fast end-to-end image to 3D skeleton pipeline. \\
\noindent\textbf{Kinematic skeleton. } We use a kinematic skeleton with 40 rotational angles as depicted in figure \ref{fig:skeleton40dof}. The bounds of the rotational articulation that we used are shown in appendix \ref{appendix:limit_joints}. The differential articulated motion angles were clamped to $\pm 5\deg$. \\
\noindent\textbf{Initialization. }
To initialize our algorithm with a 3D skeleton pose, we use a single image based approach \cite{zhou_towards_2017} across a couple of initial frames and then average the obtained skeleton as was suggested in \cite{mehta_monocular_2016}.\\
\noindent\textbf{Input 2D motion. }We use 12 detected skeleton's point: arms (shoulders, elbows, wrists) and legs (hips, knees, ankles).  The head is also included even if we experimentally find that the head detection is not stable from profile views. Since the porposed method is a 2D to 3D approach, we use ground-truth 2D joints for all the compared methods. In a wild context, pose detection algorithm as openpose \cite{cao_realtime_2017} can be used.\\
\noindent\textbf{Determination of limb lengths. }The initial pose estimate allows us to determine the limb lengths to reconstruct the 3D skeleton. In order to robustify this bone estimate, we run single image reconstruction on a couple of five first frames to average the bone length estimation. If the height of the person is available we can scale the limb lengths to Euclidean space. If there is no metric measure available on the skeleton, then the reconstruction is made up to a scale factor.\\
\begin{figure}
	\centering{
		\includegraphics[width=0.35\textwidth]{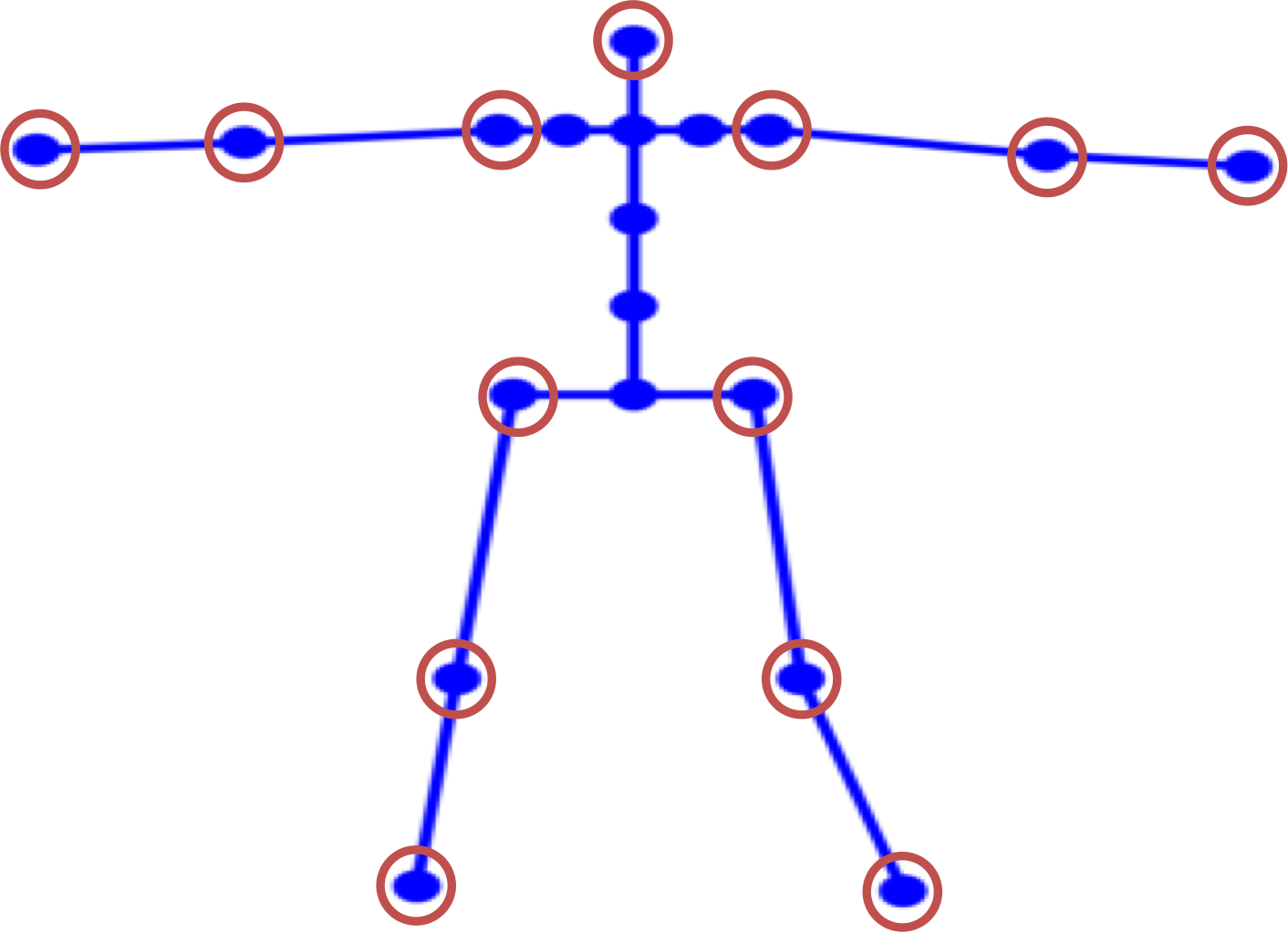}
		\caption{\label{fig:skeleton40dof} Skeleton with 40 angular articulated joints (40-DoF skeleton) which is used in our experiments. This configuration has two 3-DoF rotations in the spinal cord besides the center hip and the neck. It has two 3-DoF rotations in the mid-shoulder mimicking the scapula. The encircled joints are assumed to be detected at every frame using \textsc{openpose} \cite{cao_realtime_2017}. }}	
\end{figure}
\section{Experimental Evaluation}

\subsection{Evaluation on Synthetic Data}
To simulate plausible skeleton motion we take 20.000 frames from \Human. dataset~\cite{ionescu_human3.6m:_2014}. To cancel the noise in the kinematic modeling, we compute the inverse kinematic corresponding to each pose in every frame and obtain the articulated angles. We obtain the noise-free pose by processing back the direct kinematic model on the obtained articulated angles. To have ideal 2D projection we backproject the kinematic-noise-free poses on the corresponding frames of the original dataset. We compute the difference between successive frames in terms of rigid motion, articulated motion and 2D landmark motion to obtain synthetic ground-truth on $\RigMot$, $\ArtMot$ and $\Obs$ respectively. We take sequences from Discuss, Eat, Phone and SitDown sessions with all the subjects and all the different point of views. The experimental setup consists in varying the size of the support of the differential articulated motion and the image noise to evaluate the accuracy and the specificity of the proposed method. For a given support size, we randomly choose among the differential articulated motion the elements that do not belong to the support and set them to zero. We compute the updated pose with the direct kinematic model and compute the 2D landmark motion by projecting the successive poses and obtaining the difference. We add centered Gaussian noise on the 2D differential motion with varying standard deviation $\delta\in\{1, 2, 3, 4\}$ pixels (we remind that the focal length is equal to $1145$ in \Human.).
We run the proposed relaxed formulation \ref{eq:S1-RelaxedProb} with ADMM implementation on the constructed synthetic data to show the ability of the proposed approach to retrieve the support of the sparse differential articulated motion. We compare the results to the usage of an $\ell_2$ standard approach that can be formalized as follows
\begin{align}
	\label{eq:S1-L2}
	\tag{L2}
	\eDepl \in &\kargmin_{\tDepl} \kvvbar{\tDepl}_2\\\nonumber
	\text{s.t. }&\Obs=\MOp\,\bf{\boldsymbol{\Gamma}}\tRigMot+\MOp\,\JacSkel\,\tDepl.
\end{align}
Figure \reffig{fig:synth_data} summarizes the accuracy and the specificity scores.
It reveals that the rate of exact recovery of the locations of zeros and non-zeros in $\ArtMot$ is high up to size $9$ from a total of $40$ elements in $\ArtMot$ with a Gaussian noise of $\delta=1$ pixel. The rejection of elements wrongly classified as zeros or non-zeros is consequently very high. The specificity plot reveals that the rate of elements wrongly classified as non-zero is almost zero. It means that the proposed approach is strong to reject non-zero elements but tend to do not sufficiently reject wrong zero elements in presence of noise. The experiment with \ref{eq:S1-L2} exhibits the fact that with sparse motion, the standard $\ell_2$-norm approach is very poor in rejecting wrong zero elements (revealed in specificity plot) and thus cannot retrieve ground-truth support of sparse differential articulated motions.
\begin{figure*}[htbp]
	\includegraphics[width=0.24\textwidth]{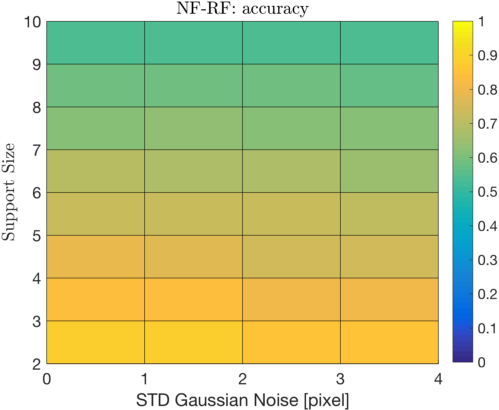}	
	\includegraphics[width=0.24\textwidth]{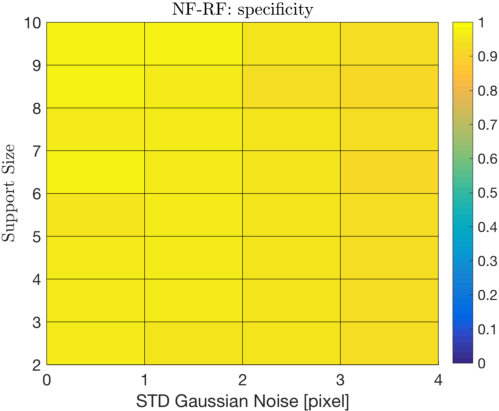}	
	\includegraphics[width=0.24\textwidth]{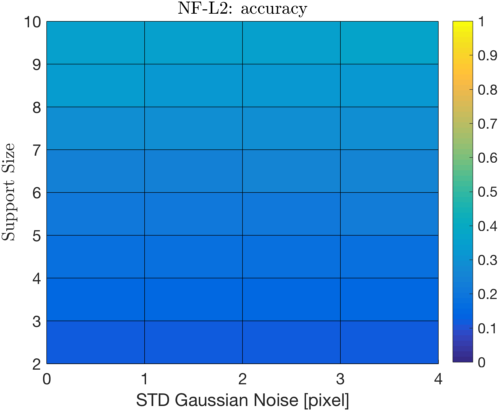}	
	\includegraphics[width=0.24\textwidth]{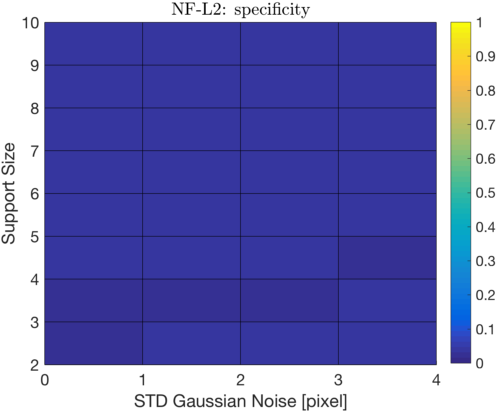}
	\caption{\label{fig:synth_data}Comparing accuracy and specificity between (\ref{eq:S1-RelaxedProb}) and (\ref{eq:S1-L2}).}	
\end{figure*}	
\subsection{Evaluation on Real Dataset}
\noindent\textbf{Dataset. }We use $3$ publicly available real dataset for the evaluation of the proposed approach: {\em{(i)}}~\Human. We evaluate on all actions for subject $9$ and $11$ \cite{ionescu_human3.6m:_2014}. The other subjects being used by the deep-learning methods for training are then not used. As was reported in previous works \cite{tome_lifting_2017,pavlakos_coarse--fine_2017,bogo_keep_2016}, we use Protocol 1 which uses all the cameras. We do not down-sample the videos and use the raw $50$ fps since it is more appropriate for assuming sparse differential articulation motion from frame-to-frame. {\em{(ii)}}~\Panoptic. We use the pose 1 video scenarios recorded at $30$ fps from HD cameras \cite{joo_panoptic_2015}. We show that at this frame rate the human motion is still slow enough to keep valid our assumption.
{\em{(iii)}}~\MPIINF. We evaluate our approach on 20,000 frames recorded at $60$ fps~\cite{mehta_monocular_2016}. 

\noindent\textbf{Compared methods.}
We use 6 state-of-the art methods with which we compare the proposed approach. These methods are either single image or image sequence based approaches as well as either deep learning or basis shapes based approaches. {\CBSS} is a single image based approach relying on basis shapes.
{\CBSV} is an image sequence based approach using also basis shapes.
{\LfD}, {\LCR} and {\SEB} are  deep learning based approaches using a single image.
{\VNECT} is a deep learning based approach using an image sequence. This method is the gold standard method with which we draw more specific comparison. The methods are all using weak perspective models, the 3D reconstructions are thus transformed to perspecitve for comparison as was done in \cite{mehta_monocular_2016}.\\
\noindent\textbf{Qualitative evaluation. }
Figures \ref{fig:qualitative-h36m2} to \ref{fig:qualitative-mpii2} show qualitative reconstructions on classic and challenging poses. We display only the results from the proposed method and \VNECT. On \Human~we display in figure \ref{fig:qualitative-h36m2} two classic poses where only part of the limb moves from frame-to-frame. On \Panoptic~we show a challenging pose where the camera is top view and the skeleton is undergoing unusual poses with bending knees. Figure \ref{fig:qualitative-panoptic2} show that the spine for instance with the proposed model is curved as the in the real pose while the reconstruction from \VNECT~shows straight spine. On \MPIINF~we show in figure \ref{fig:qualitative-mpii2} a posture where the bottom body moves.\\
\begin{table*}
	\begin{center}
		\begin{tabularx}{\textwidth}{lRRRRRRRRRRRRRRRR}
			\toprule
			Method&Direct&Discuss&Eat&Greet&Phone&Posing&Purch.&Sit&Sit Down&Smoke&Take Photo&Wait&Walk&Walk Dog&Walk Pair&All\\
			\hline
			\CBSS&87.3&110.1&85.4&102.4&118.3&107.4&100.1&127.1&201.2&105.9&141.2&117.1&80.4&115.4&98.8&114.1\\
			\SEB&63.1&75.3&64.9&61.1&92.5&94.1&106.7&113.5&92.8&110.6&76.5&87.9&68.9&122.4&53.1&74.7\\
			\LfD&64.5&74.4&75.6&85.8&86.4&68.7&97.2&110.1&173.7&84.9&111.2&85.3&72.1&92.2&73.9&87.9\\
			\LCR&76.4&81.1&74.9&83.2&91.1&80.1&72.3&106.8&126.5&87.5&105.9&82.5&64.7&86.5&83.4&86.5\\
			\hline
			\CBSV&86.1&109.7&86.5&101.1&117.1&106.1&99.8&125.2&200.1&103.9&142.1&116.2&79.3&114.4&97.5&113.5\\
			\VNECT&62.7&78.3&64.1&73.3&89.4&64.7&76.3&110.7&145.6&80.1&94.6&74.5&54.6&82.5&60.4&81.7\\
			\hline
			\OUR&65.9&83.7&80.4&86.7&101.9&88.2&91.0&\textbf{86.2}&\textbf{90.4}&86.8&81.8&82.5&80.7&87.6&78.9&\\
			\bottomrule
		\end{tabularx}
		\caption{\label{tab:human3.6M} Evaluation of the proposed method on \Human.} dataset following Protocol 1. The comparisons were performed on subjects $9$ and $11$. 
	\end{center}
\end{table*}

\begin{table}[htbp]
	\begin{tabularx}{0.49\textwidth}{lRR}
		\toprule\\
		&MPJPE&Reconst. Error\\
		\hline
		\CBSS&60.1&57.6\\
		\SEB&51.64&50.5\\
		\LfD&200.7&55.1\\
		\hline
		\CBSV&56.9&53.1\\
		\VNECT&29.3&23.9\\
		\hline
		\OUR&\textbf{26.8}&\textbf{19.8}\\
		\bottomrule
	\end{tabularx}
	\caption{\label{tab:panoptic} Evaluation of the proposed method on \Panoptic.}
\end{table}

\begin{table}[htbp]
	\begin{tabularx}{0.49\textwidth}{lRR}
		\toprule\\
		&MPJPE&Reconst. Error\\
		\hline
		\CBSS&60.1&57.6\\
		\SEB&51.64&50.5\\
		\LfD&200.7&55.1\\
		\hline
		\CBSV&56.9&53.1\\
		\VNECT&34.6&22.3\\
		\hline
		\OUR&28.9&23.14\\
		\bottomrule
	\end{tabularx}
	\caption{\label{tab:mpiinf} Evaluation of the proposed method on \MPIINF.}
\end{table}

\noindent\textbf{Joint occlusion. }
We test the effect of occluding one joint at a time during the whole sequence. We run this experiment on the \Panoptic~sequence (pose $1$, id=$141216$ with cam HD $2$). Figure~\ref{fig:occluded_joint} shows the mean and std errors (the mean being the MPJPE error). It appears that there is no specific joint for which occlusion has dramatic consequence on the skeleton reconstruction. The occlusion of end joints do not have effect on its parent joint. For instance, the occlusion of the wrist involves no motion of the elbow joint.\\
\noindent\textbf{Failure cases. } As every tracking method, failures occur after a certain number of frames. In our case, the two main reasons are the approximations in bone lengths and the location of 2D joints. These two measures can never be exactly estimated. Taking into account the noise in these two estimates could improve the horizon of tracking. This is out of the context of current paper and will be done in futur work. In these experiments we re-compute an initialzation pose when the reprojection error is above $50$ pixels. We observed that this is necessary for every 300 frames in average for a $30$ fps video.\\

\noindent\textbf{Discussion and comments. }

In our experimental validation, we have observed occasional outliers affecting the MPJPE which are mainly due to  bias in observation. Indeed, the detected joint either jitter or is not well aligned with the kinematic joint. The deep learning methods have less outliers in MPJPE because during the training step, the network learns the systematic bias between the observation and the 3D output. These methods give good reconstructions mainly  in the range of the trained data. They are  dynamically less responsive than kinematic-based model and more specifically less suited to unusual poses and motions. We also observe that deep learning methods are less good in recovering fine-grained  orientation with respect to camera.

Constant limb length prior penalizes kinematic-based approaches even if anatomically the limb length constraint is meaningful. On the one hand, ground-truth setup (either with or without markers) cannot ensure that this constraint is filled as is shown in figures \ref{fig:gbone_lengths}, which gives biased ground-truth. On the other hand, detection methods cannot provide guaranty of giving joints that are aligned with the kinematic model. In this case, a single frame approach can provide smaller errors like is the case with \cite{tome_lifting_2017}. Pre-processing on the detected joints (noise filtering, smoothing, etc) can help image sequence or video based methods to reduce the reconstruction errors \cite{mehta_vnect:_2017}.

Figure \ref{fig:fraction_incorrect_joint} draws the fraction of joints exceeding error thresholds for MPJPE as was done previously in \cite{mehta_vnect:_2017}. We can observe that the proposed method has similar rates of errors as in \VNECT~but without smoothing the 2D detected joints. This enhances the fact that sparsity of articulated skeleton's motion is justified in a fast enough frame rates besides the fact that $\ell_1$ norm minimization performs better than $\ell_2$-norm minimization in rejecting outliers. 

\section{Conclusion}
In this paper, we formalized the 3D human motion problem using sparse articulated motion as prior. We used two formulation; an ideal formulation with $\ell_0$-norm and a relaxed formulation with $\ell_1$-norm. We defined the PKSP condition (Projective Kinematic Sparse Property) that establishes the set of possible ambiguities between differential rigid and articulated motion. We established that filling condition allows us to retrieve the ground-truth 3D human, solution of the ideal formulation, only by solving the relaxed $\ell_1$-norm problem. We provided extensive results on synthetic and real datasets. We showed that the proposed formulation performs as good as major 3D human methods without extensive preprocessing on the data. As futur work, we will integrate model noise and image noise in order to robustify the reconstruction method regarding bone length estimation and joint detection.
\begin{figure*}
	\begin{tabular}{cccccc}
		\includegraphics[width=0.15\textwidth]{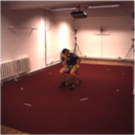}&
		\includegraphics[width=0.15\textwidth]{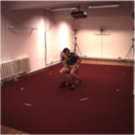}&
		\includegraphics[width=0.15\textwidth]{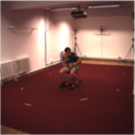}&
		\includegraphics[width=0.15\textwidth]{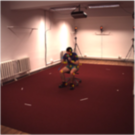}&
		\includegraphics[width=0.15\textwidth]{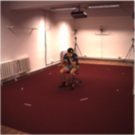}&
		\includegraphics[width=0.15\textwidth]{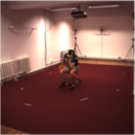}\\
		\includegraphics[width=0.15\textwidth]{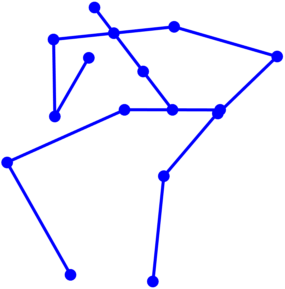}&
		\includegraphics[width=0.15\textwidth]{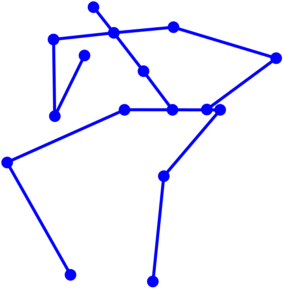}&
		\includegraphics[width=0.15\textwidth]{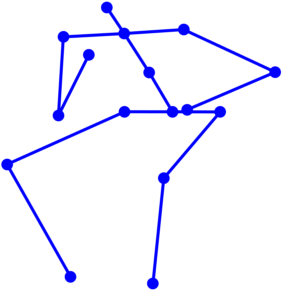}&
		\includegraphics[width=0.15\textwidth]{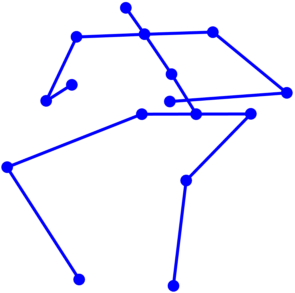}&
		\includegraphics[width=0.15\textwidth]{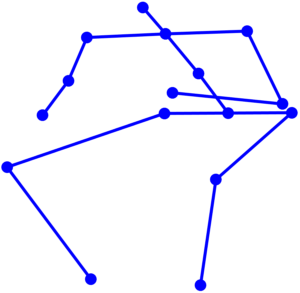}&
		\includegraphics[width=0.15\textwidth]{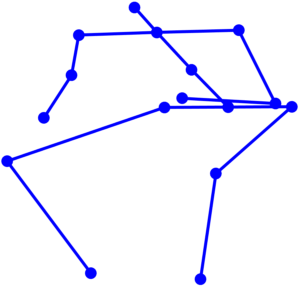}\\
		\includegraphics[width=0.15\textwidth]{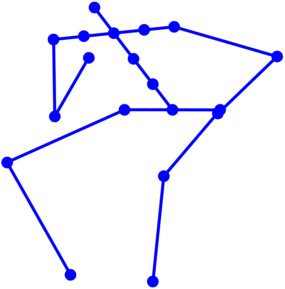}&
		\includegraphics[width=0.15\textwidth]{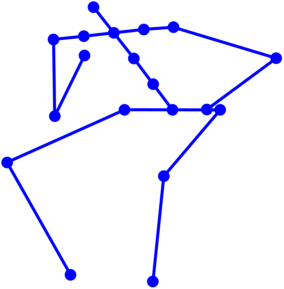}&
		\includegraphics[width=0.15\textwidth]{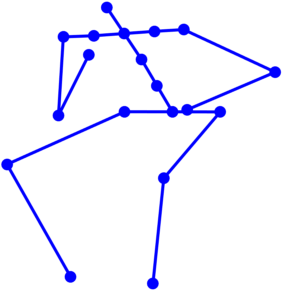}&
		\includegraphics[width=0.15\textwidth]{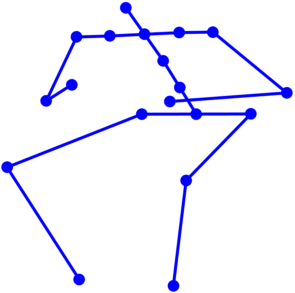}&
		\includegraphics[width=0.15\textwidth]{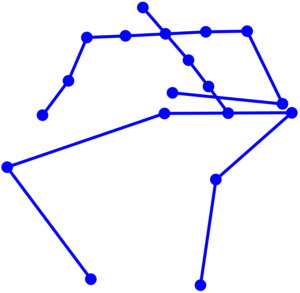}&
		\includegraphics[width=0.15\textwidth]{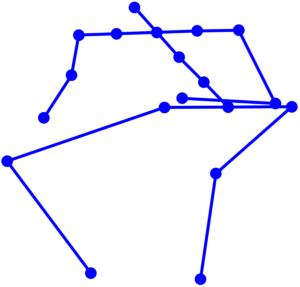}
		
	\end{tabular}
	\caption{\label{fig:qualitative-h36m2}  \Human.~3D reconstructions (Sitting).}	
\end{figure*}

\begin{figure*}
	\begin{tabular}{cccc}
		\includegraphics[width=0.25\textwidth]{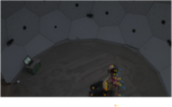}&
		\includegraphics[width=0.25\textwidth]{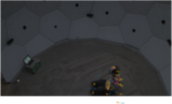}&
		\includegraphics[width=0.25\textwidth]{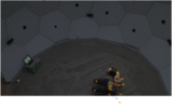}&
		\includegraphics[width=0.25\textwidth]{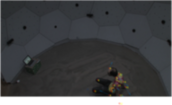}\\
		\includegraphics[width=0.15\textwidth]{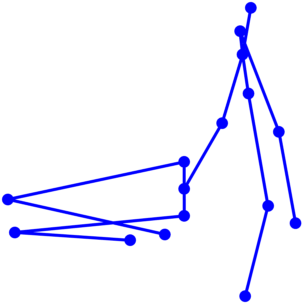}&
		\includegraphics[width=0.15\textwidth]{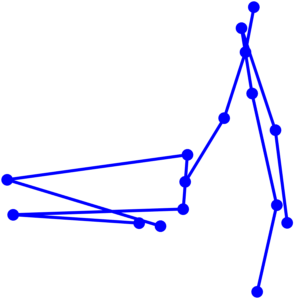}&
		\includegraphics[width=0.15\textwidth]{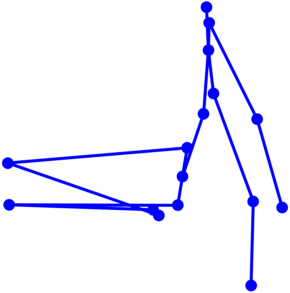}&
		\includegraphics[width=0.15\textwidth]{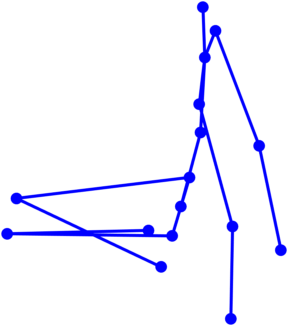}\\
		\includegraphics[width=0.15\textwidth]{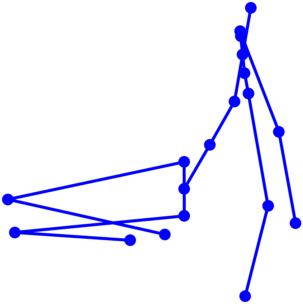}&
		\includegraphics[width=0.15\textwidth]{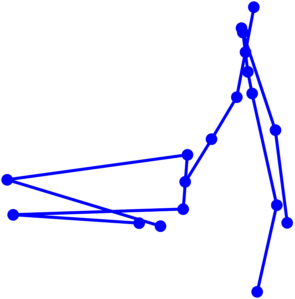}&
		\includegraphics[width=0.15\textwidth]{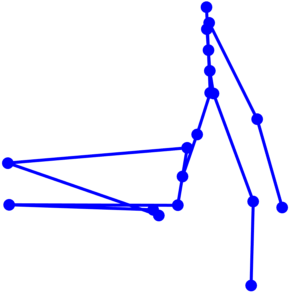}&
		\includegraphics[width=0.15\textwidth]{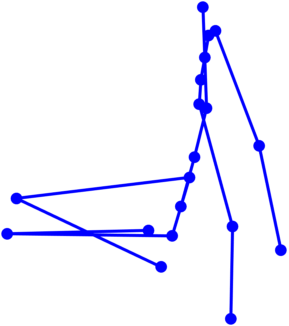}\\
	\end{tabular}
	\caption{\label{fig:qualitative-panoptic2}  \Panoptic.~3D reconstructions (Pose~1, id=141216 with cam HD 2).}	
\end{figure*}

\begin{figure}[htbp]
	\begin{tabular}{cc}
		\includegraphics[width=0.2\textwidth]{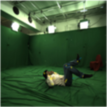}&
		\includegraphics[width=0.2\textwidth]{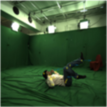}\\
		\includegraphics[width=0.14\textwidth]{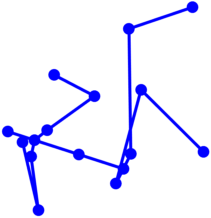}&
		\includegraphics[width=0.14\textwidth]{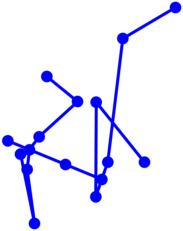}\\
		\includegraphics[width=0.15\textwidth]{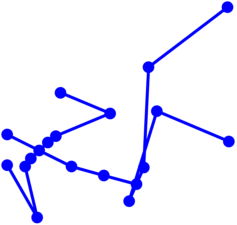}&
		\includegraphics[width=0.15\textwidth]{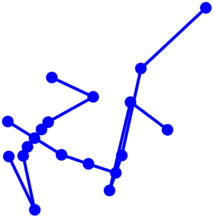}\\	
				\includegraphics[width=0.2\textwidth]{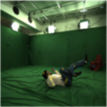}&
				\includegraphics[width=0.2\textwidth]{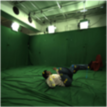}\\
				\includegraphics[width=0.14\textwidth]{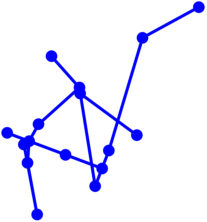}&
				\includegraphics[width=0.13\textwidth]{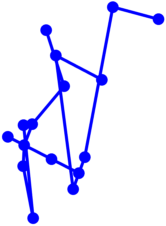}\\
				\includegraphics[width=0.15\textwidth]{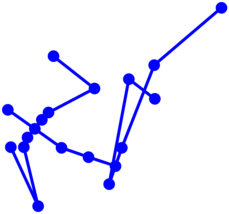}&
				\includegraphics[width=0.15\textwidth]{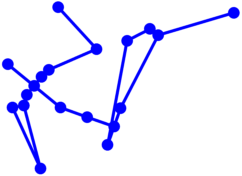}\\	
	\end{tabular}
	\caption{\label{fig:qualitative-mpii2}  \MPIINF.~3D reconstructions (S3, Seq2, Video 5).}	
\end{figure}
\begin{figure}[htbp]
	\includegraphics[width=0.23\textwidth]{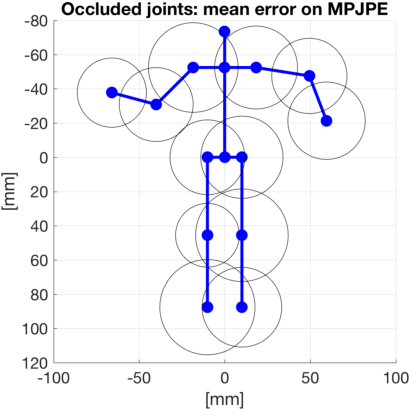}
	\includegraphics[width=0.23\textwidth]{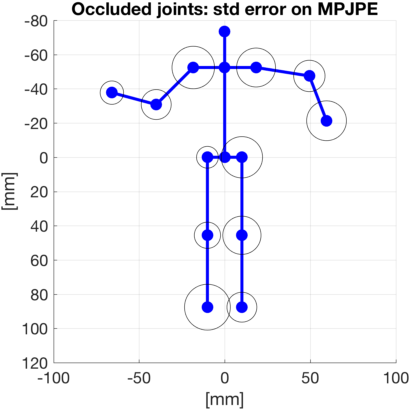}
	\caption{\label{fig:occluded_joint} 3D error evaluation with occluded joints. The center of circles represent the joint that is considered to be occluded. The radius of circles represent the related error computed on the whole skeleton reconstruction (mean and std respectively from left to right).}	
\end{figure}
\begin{figure}
	\includegraphics[width=0.45\textwidth]{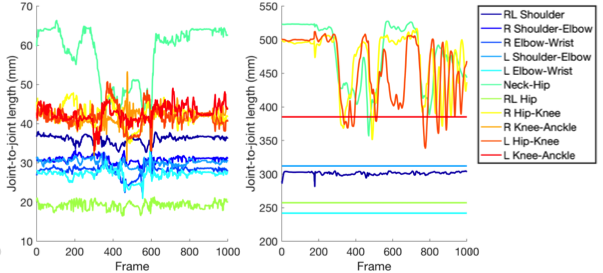}
	\caption{\label{fig:gbone_lengths}  Sample ground-truth joint-to-joint lengths. From left to right: \Panoptic (pose~$1$, $141216$) and \MPIINF.  (S$3$, Seq$2$, video $5$).}	
\end{figure}

\begin{figure}
	\begin{center}
		\includegraphics[width=0.3\textwidth]{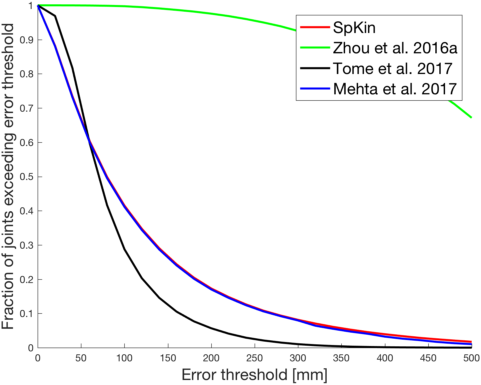}
		\caption{\label{fig:fraction_incorrect_joint} Fraction of joints exceeding error threshold with \Human dataset.}
	\end{center}	
\end{figure}

\appendix

\section{Bounds of rotational ariticulations\label{appendix:limit_joints}}

\begin{table*}[htbp]
	\begin{tabularx}{\textwidth}{lRlR}
		Articulation&Rotation Angles&Min $\kbracket{\deg}$&Max $\kbracket{\deg}$\\
		\hline
		Hip&$\kparen{R_z,R_x, R_y}$&$\kparen{-30,-10,-5}$&$\kparen{30,180,5}$\\
		Right Hip&$\kparen{R_z,R_x, R_y}$&$\kparen{-180,-170,-90}$&$\kparen{180,90,90}$\\
		Right Knee&$\kparen{R_x}$&$\kparen{0}$&$\kparen{150}$\\
		Spine 1&$\kparen{R_z,R_x, R_y}$&$\kparen{-2,0,-2}$&$\kparen{2,15,2}$\\
		Spine2&$\kparen{R_z,R_x, R_y}$&$\kparen{-2,0,-2}$&$\kparen{2,15,2}$\\
		Neck &$\kparen{R_z,R_x, R_y}$&$\kparen{-60,-30, -80}$&$\kparen{60,80, 80}$\\
		Mid Shoulder&$\kparen{R_z,R_x, R_y}$&$\kparen{-5,-5,-5}$&$\kparen{5,5,5}$\\
		Right Shoulder&$\kparen{R_z,R_x, R_y}$&$\kparen{-180,-100,-90}$&$\kparen{180,90,90}$\\
		Right Elbow&$\kparen{R_x}$&$\kparen{-150}$&$\kparen{0}$\\
		\bottomrule
	\end{tabularx}
	\caption{\label{tab:limit-joints} \small{Bounds of the rotational angles for the 40 DoF skeleton model. $R_x$ and $R_y$ are opposite signs for the angular articulations of the left skeleton side.}}
\end{table*}

\end{document}